\documentclass[runningheads]{llncs}
\usepackage{graphicx}
\usepackage{amsmath,amssymb} 

\usepackage{color}
\usepackage{epsfig}
\usepackage{graphicx}
\usepackage{amsmath}
\usepackage{amssymb}
\usepackage{booktabs}
\usepackage{multirow}
\usepackage{tabularx}
\usepackage{booktabs}
\usepackage{makecell}
\usepackage{url}
\usepackage[colorlinks,linkcolor=blue]{hyperref}
\usepackage{comment}
\usepackage{subcaption}
\usepackage{ulem}
\usepackage[ruled,vlined]{algorithm2e}

\usepackage[width=122mm,left=12mm,paperwidth=146mm,height=193mm,top=12mm,paperheight=217mm]{geometry}
\begin{document}
\pagestyle{headings}
\mainmatter
\def\ECCVSubNumber{2798}  
\newcommand{\cy}[1]{\textcolor{red}{#1}}
\newcommand{\hl}[1]{\textcolor{blue}{#1}}
\title{FakeCLR: Exploring Contrastive Learning for Solving Latent Discontinuity in Data-Efficient GANs} 
\author{Ziqiang Li\inst{1}\orcidID{0000-0001-9484-2310}
\thanks{This work was performed when Ziqiang Li was visiting JD Explore Academy as a research intern.}
\and
Chaoyue Wang\inst{2} \and
Heliang Zheng\inst{2}\and
Jing Zhang\inst{3}\and
Bin Li\inst{1}
}
%
%
\institute{University of Science and Technology of China, China \and
JD Explore Academy, China \and
The University of Sydney, Australia\\
\email{iceli@mail.ustc.edu.cn}, \email{chaoyue.wang@outlook.com}, \\
\email{zhenghl@mail.ustc.edu.cn}, \email{jing.zhang1@sydney.edu.au}, \email{binli@ustc.edu.cn}}

\titlerunning{FakeCLR for Data-Efficient GANs} 
\maketitle

\begin{abstract}

Data-Efficient GANs (DE-GANs), which aim to learn generative models with a limited amount of training data, encounter several challenges for generating high-quality samples. Since data augmentation strategies have largely alleviated the training instability, how to further improve the generative performance of DE-GANs becomes a hotspot. Recently, contrastive learning has shown the great potential of increasing the synthesis quality of DE-GANs, yet related principles are not well explored. In this paper, we revisit and compare different contrastive learning strategies in DE-GANs, and identify (i) the current bottleneck of generative performance is the discontinuity of latent space; (ii) compared to other contrastive learning strategies, Instance-perturbation works towards latent space continuity, which brings the major improvement to DE-GANs. Based on these observations, we propose FakeCLR, which only applies contrastive learning on perturbed fake samples, and devises three related training techniques: Noise-related Latent Augmentation, Diversity-aware Queue, and Forgetting Factor of Queue. Our experimental results manifest the new state of the arts on both few-shot generation and limited-data generation. On multiple datasets, FakeCLR acquires more than $15\%$ FID improvement compared to existing DE-GANs. Code is available at \href{https://github.com/iceli1007/FakeCLR}{https://github.com/iceli1007/FakeCLR}.

\end{abstract}
\section{Introduction}
Generative Adversarial Networks~\cite{goodfellow2014generative} are capable of generating realistic images~\cite{karras2020analyzing} indistinguishable from real ones, which acquire remarkable achievements in image processing~\cite{wang2018perceptual,wenlong2021ranksrgan,Yang_2022_CVPR,qiu2019world} and computer vision~\cite{yang2018pose,tao2019resattr,jiang2020tsit,qiu2021scene}. The success of GANs relies on large-scale training datasets, for example, CelebA consists of around 200K images, and LSUN has more than 1M images. However, many scenarios may have a limited amount of training data, \textit{e.g.} there are just 10 examples per artist in the Artistic-Faces dataset~\cite{yaniv2019face}. 
Therefore, Data-Efficient Generative Adversarial Networks (DE-GANs \cite{li2022comprehensive}) are important to many real-world practices and attract more and more attention. More analyses can be found in the survey \cite{li2022comprehensive}.

\begin{figure*}[t!]
\setlength{\abovecaptionskip}{0.1cm}
    \setlength{\belowcaptionskip}{-0.3cm}
	\centering
	\includegraphics[scale=0.45]{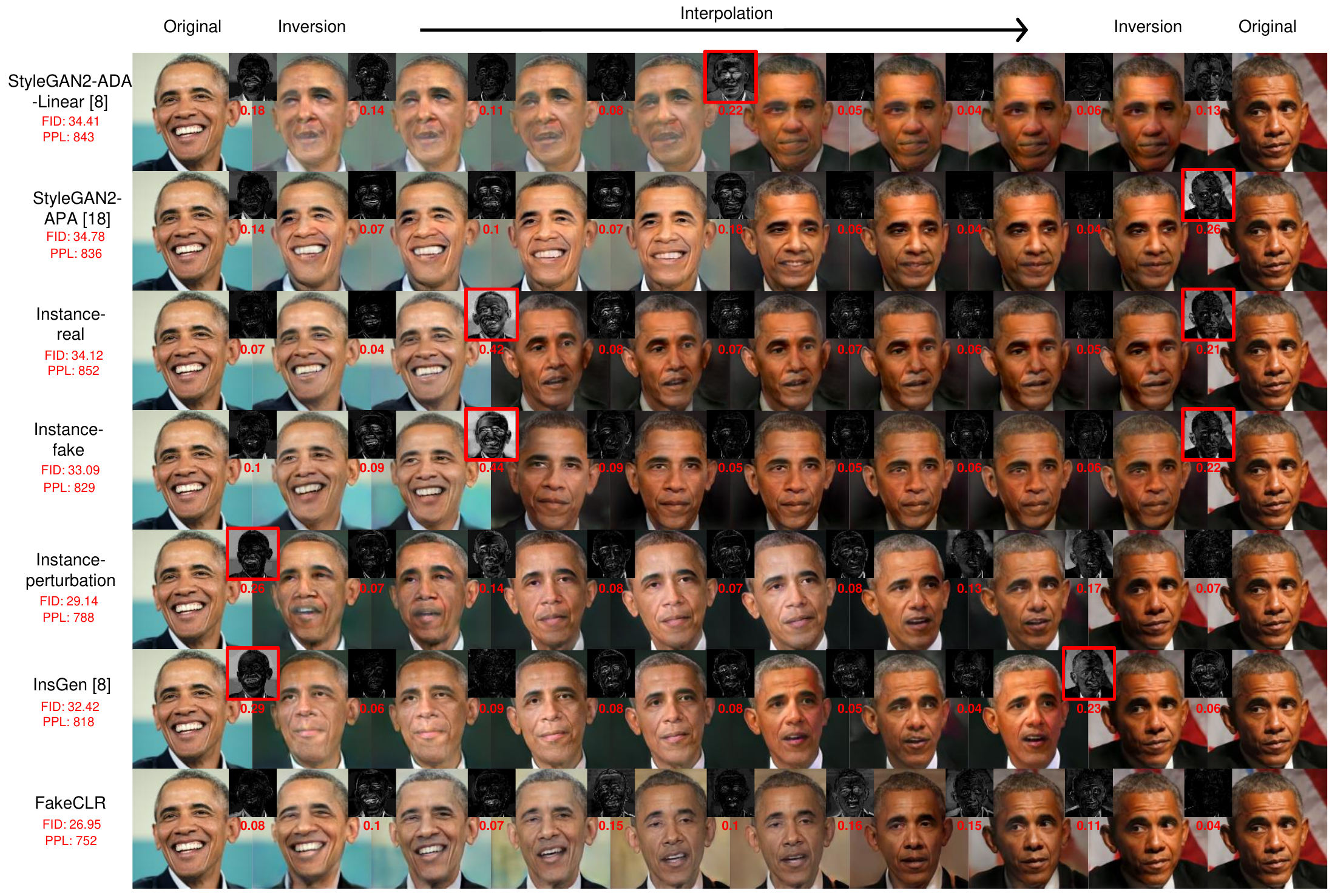}
	\caption{{Training GANs with as little as 100 training samples (Obama dataset) typically results in severe discontinuity in latent space, \textit{i.e.} under-diversity interpolation on the top four rows. StyleGAN2-ADA-Linear \cite{yang2021data} is the baseline model with data augmentation. Generator trained with Instance-perturbation and ours show more accurate inversion, smoother latent space, more diverse interpolation, and better FID and PPL. The small grey images visualize the difference between the two face images. The red numbers are mean of pixel-wise difference, we highlighted the difference score$>0.2$ with \cy{ red} border.}}
	\label{FIG:motivation}
\end{figure*}

Previous studies~\cite{zhao2020differentiable,karras2020training} demonstrate that severe discriminator overfitting leads to the training failure of DE-GANs. The less training data there is, the earlier the discriminator overfitting happens. Specifically, the discriminator is overly confiding in a few real samples and shows abnormal decision boundaries. Then, the gradients passing to generator are inaccurate, and the training instability problems occur \cite{arjovsky2017towards}. To mitigate it, many technologies~\cite{yang2021one,sauer2021projected,kumari2021ensembling,chen2021data,tseng2021regularizing} have been proposed. Among them, data augmentation strategies~\cite{zhao2020differentiable,karras2020training,tran2021data,jiang2021deceive} largely alleviate the training instability issue.



Currently, how to improve the diversity and quality of DE-GANs generated images becomes the major objective of the community. Contrastive learning (CL) \cite{he2020momentum,chen2020simple}, as a representative self-supervised learning algorithm, has been applied to different components of GANs, \textit{e.g.} real or fake samples. Although existing methods demonstrate promising performance on DE-GANs, the motivation and principle behind different CL strategies are not well explored and justified. In this paper, we revisit the principle of different contrastive learning strategies and explore a novel perspective of understanding the connections between contrastive learning and DE-GANs' training.

In particular, we experimentally analyze the results of thee different contrastive learning strategies {(Instance-real: Eq (\ref{eq:instance-real}), Instance-fake: Eq (\ref{eq:instance-fake}), and Instance-perturbation: Eq (\ref{eq:instance-perturbation}))} on DE-GANs. Among them, Instance-perturbation contributes most and achieves significant gains in terms of FID, while others bring slight or even negative improvement (Table \ref{Tab:Insgen_ablation_1}). To dig out how Instance-perturbation works, we analyze the latent space by GAN inversion and interpolation. As shown in Fig.~\ref{FIG:motivation} (top four rows), the training data of DE-GANs is sparse and discrete, and the generator tends to remember these samples, which leads to the discontinuity of the latent space. It can be observed that Instance-perturbation (the fifth row) can alleviate this problem, which achieves better inversion results and more continuous latent interpolation. Such results indicate that current DE-GANs, \textit{e.g.} the most representative method StyleGAN2-ADA, are troubled by the discontinuities on latent space, and reasonable contrastive learning methods may help to solve this problem. 
Based on these findings, we propose a novel contrastive learning method termed FakeCLR for DE-GANs, which only applies Instance-perturbation and additional three exquisite strategies 
inspired by continuously changing fake images. These three strategies are Noise-related Latent Augmentation, Diversity-aware Queue, and Forgetting Factor of Queue. Finally, the proposed FakeCLR achieves state-of-the-art performance and contributes to more continuous latent space. Overall, our contributions include:
\begin{itemize}
	\item \textit{Illustrating the major bottleneck of DE-GANs.} We explore and emphasize the bottleneck, discontinuous of the latent space, of existing DE-GAN methods.
	\item \textit{Exploring the connections between contrastive learning and DE-GANs.} Through comprehensive experiments, we revisit three popular contrastive learning strategies that can be applied to DE-GANs, and identify that only Instance-perturbation brings the major improvement on generative performance.
	\item \textit{A new contrastive learning method for DE-GANs.} We propose new contrastive learning with noise perturbation on fake images for DE-GANs called FakeCLR. Meanwhile, three technical innovations have been proposed for improving generative performance.
	\item \textit{Comprehensive experimental validations on various training settings.} Besides data augmentation, the proposed FakeCLR achieves significant improvements on both limited-data generation and few-shot generation tasks.

\end{itemize}

\section{Background}
\subsection{Data-Efficient Generative Adversarial Networks}
Generative Adversarial Networks (GANs) \cite{goodfellow2014generative} act a two-player adversarial game, where the generator $G(z)$ is a distribution mapping function that transforms low-dimensional latent distribution $p_z$ to target distribution $p_g$. And the discriminator $D(x)$ evaluates the divergence between the generated distribution $p_g$ and real distribution $p_r$. The generator and discriminator minimize and maximize the 
adversarial loss function.
This min-max game can be expressed as:
\begin{equation}
	\begin{aligned}
		\min _{\phi} \max _{\theta} f(\phi, \theta) =\mathbb{E}_{\mathbf{x} \sim p_{r}}\left[g_{1}\left(D_{\theta}(\mathbf{x})\right)\right] 
		+\mathbb{E}_{\mathbf{z} \sim p_{z}}\left[g_{2}\left(D_{\theta}\left(G_{\phi}(\mathbf{z})\right)\right)\right] ,
	\end{aligned}
\end{equation}
where $\phi$ and $\theta$ are parameters of the generator $G$ and discriminator $D$, respectively. $g_1$ and $g_2$ are different functions for various GANs. Data-Efficient GANs is a special case of GANs, which targets at obtaining a generator that can fit the real data distribution with limited, few-shot training samples. Formally, given the limited data distribution $p_l$ (Generally, samples of $p_l$ is the subset of $p_r$), we expect the generated distribution $p_g$ obtained by training with $p_l$ to be as close as possible to $p_r$. Concretely, loss functions of each network are formalized as:
\begin{equation}
	\begin{aligned}
		&\mathcal{L}_{D}=-\mathbb{E}_{\mathbf{x} \sim p_{l}}\left[g_{1}\left(D_{\theta}(\mathbf{x})\right)\right] -\mathbb{E}_{\mathbf{z} \sim p_{z}}\left[g_{2}\left(D_{\theta}\left(G_{\phi}(\mathbf{z})\right)\right)\right], \\
		&\mathcal{L}_{G}=\mathbb{E}_{\mathbf{z} \sim p_{z}}\left[g_{2}\left(D_{\theta}\left(G_{\phi}(\mathbf{z})\right)\right)\right] .
	\end{aligned}
\end{equation}

Previous studies attribute the degradation of DE-GANs to the overfitting of the discriminator. Many data augmentation, regularization, architectures, and pre-training techniques \cite{li2022comprehensive} have been proposed to mitigate this issue. 
(i) Data Augmentation \cite{zhao2020differentiable,zhao2020image,karras2020training,8627945,tran2020towards} is a striking method for mitigating overfitting of the discriminator and orthogonal to other ongoing researches on training, architecture, and regularization. Popular augmentation strategies, such as Adaptive Data Augmentation (ADA), are employed as the essential complements in most DE-GANs. (ii) Regularization \cite{li2020systematic,kong2021smoothing,yang2021data} is also a kind of popular technology for enhancing generalization in discriminator by introducing priors or extra supervision tasks. For instance, Tseng \textit{et al.} \cite{tseng2021regularizing} proposed an anchors-based regularization term to slow down the training of discriminator under limited data. 
(iii) Architecture is also the key point to improve the performance and stability of GANs. Some light-weight networks \cite{liu2020towards} have been employed in DE-GANs. Additionally, some model compression methods \cite{zhang2021efficient} have also been introduced in DE-GANs \cite{chen2021data}. 
(iv) Pre-training~\cite{noguchi2019image,wang2020minegan,wang2018transferring,mo2020freeze,yang2021one,sauer2021projected,kumari2021ensembling} provides another solution to decrease the demand for data. Some studies argued that generator pre-training introduces some image-diversity~\cite{grigoryev2022when} and latent-smooth~\cite{ojha2021few} priors, and discriminator pre-training~\cite{grigoryev2022when} provides accurate gradient for DE-GANs. 


According to the scale of training data, data-efficient generation can be roughly divided into two tasks: few-shot generation and limited-data generation. Few-shot generation is a challenging task due to the employed tiny amount of data, for instance, 10-shot and 100-shot. 
Compared to few-shot generation, limited-data generation
employs more data, such as 1K or 2K. Generally, 
the employed technologies, apart from data augmentation, are notably different between few-shot and limited-data generation. Our proposed FakeCLR is a universal technology achieving remarkable performance on both tasks.

\subsection{Contrastive Learning}
Contrastive learning \cite{oord2018representation,chen2020simple,he2020momentum} is a kind of self-supervised method to extract representations. 
Generally, given an image $\mathbf{x}$, two random views (query: $\mathbf{x}_{q}$ and key: $\mathbf{x}^+_{q}$) are created by different data augmentations ($T_1$ and $T_2$), defined as:
\begin{equation}
	\mathbf{x}_{q}=T_1(x), \quad \mathbf{x}^+_{q}=T_2(x) .
\end{equation}
Positive pair is defined as such query-key pair, between $\mathbf{x}_{q}$ and $\mathbf{x}^+_{q}$, from the same image. Negative pairs are defined as pairs from different images, \textit{i.e.} between $\mathbf{x}_{q}$ and $
\left\{\mathbf{x}^-_{k_{i}}\right\}_{i=1}^{N}$, where $\mathbf{x}^-_{k_{i}}=T_2(\mathbf{x}_i)$ and $\mathbf{x}\neq\mathbf{x}_i$, $i=1,\cdots,N$. All views are passed through the feature extractor $F(\cdot)$ to acquire representation $\mathbf{v}$:
\begin{equation}
	\begin{gathered}
		\mathbf{v}_{q}=F\left(\mathbf{x}_{q}\right), \quad \mathbf{v}^+_{q}=F\left(\mathbf{x}^+_{q}\right), \quad \mathbf{v}^-_{k_{i}}=F\left(\mathbf{x}^-_{k_{i}}\right), \quad i=1 \ldots N. \\
	\end{gathered}
\end{equation}
Contrastive learning aims to maximize the similarity of positive pairs and makes negative pairs dissimilar. Therefore, the InfoNCE loss \cite{oord2018representation} is designed as:

\begin{equation}
\small
	\begin{split}
		\mathcal{C}_{F(\cdot), \phi(\cdot)}  & \left(\mathbf{x}_{q}, \mathbf{x}^+_{q}, \left\{\mathbf{x}^-_{k_{i}}\right\}_{i=1}^{N} \right)=\\
		 &-\log\frac{\exp \left(\phi\left(\mathbf{v}_{q}\right)^{T} \phi\left(\mathbf{v}^+_{q}\right) / \tau\right)}{\exp \left(\phi\left(\mathbf{v}_{q}\right)^{T} \phi\left(\mathbf{v}^+_{q}\right) / \tau\right)+\sum_{i=1}^{N} \exp \left(\phi\left(\mathbf{v}_{q}\right)^{T} \phi(\mathbf{v}^-_{k_{i}}) / \tau\right)},
	\end{split}
\end{equation}
where $\phi(\cdot)$ is the head network to project the representation to different spaces, and $\tau$ is the temperature coefficient. 
Generally, contrastive learning strategies are employed in discriminative models and have achieved convincing performance on representation leraning~\cite{chen2020simple,he2020momentum,chen2021exploring,xiao2020should}. 
There are still a few works~\cite{jeong2021training,yang2021data} to apply contrastive learning into unconditional generative tasks. Specifically, Jeong \textit{et al.} \cite{jeong2021training} proposed Contrastive Discriminator (ContraD), a way of training discriminators of GANs using improved SimCLR \cite{chen2020simple}. Yang \textit{et al.} \cite{yang2021data} proposed InsGen, a instance discrimination realized by MoCo-v2~\cite{chen2020improved}, achieving the state-of-the-art performance on limited-data generation. 
Different from them, our FakeCLR only applies contrastive learning to the fake images and adopts three advanced strategies designed specifically for generative tasks, acquiring remarkable performance on both few-shot and limited-data generation.

\section{How Contrastive Learning Benefits DE-GANs?}
Previous works tend to attribute the success of contrastive learning to extra discrimination tasks. For example, in \cite{yang2021data}, contrastive learning is considered as an instance discrimination task applying to both real and fake images by two additional heads $\phi^r(\cdot)$ and $\phi^f(\cdot)$ besides the original bi-classification head $\phi^{d}(\cdot)$. 
In this section, we explore the question of \textit{How contrastive learning benefits Data-Efficient GANs?} by 
conducting comprehensive comparisons among different contrastive learning strategies.

\subsection{Comparisons and Analyses on Contrastive Learning Strategies} 

Following existing study \cite{yang2021data}, three main contrastive learning strategies can be applied in Data-Efficient GANs, \textit{i.e.} Instance-real, Instance-fake, and Instance-perturbation. Specifically, \textit{Instance-real} introduces an independent real head $\phi^r(\cdot)$ to conduct additional instance discrimination on real images, \textit{i.e.,}
\begin{equation}
	\begin{gathered}
		\mathcal{C}^r=\mathcal{C}_{d(\cdot), \phi^r(\cdot)}\left(T_q(\mathbf{x}_{q}), T_{k^+}(\mathbf{x}_{q}),\left\{T_{k_i}(\mathbf{x}_{k_{i}})\right\}_{i=1}^{N}\right),
	\end{gathered}
	\label{eq:instance-real}
\end{equation}
where $T_q$, $T_{k^+}$, and $T_{k_i}$ are different augmentations, $\mathbf{x}_q$ and $\mathbf{x}_{k_i}$ are sampled from the real distribution. Similarly, \textit{Instance-fake} introduces an independent fake head $\phi^f(\cdot)$ to conduct instance discrimination on fake images through contrastive learning and whose objective is formalized as:
\begin{equation}
	\begin{gathered}
		\mathcal{C}^f=\mathcal{C}_{d(\cdot), \phi^f(\cdot)}\left(T_q(G(\mathbf{z}_{q})), T_{k^+}(G(\mathbf{z}_{q})),\left\{T_{k_i}(G(\mathbf{z}_{k_{i}}))\right\}_{i=1}^{N}\right),
	\end{gathered}
	\label{eq:instance-fake}
\end{equation}
where $\mathbf{z}_q$ and $\mathbf{z}_{k_i}$ are sampled from the latent distribution. Furthermore, \textit{Instance-perturbation} introduces a noise perturbation strategy \cite{cheung2020modals} to further improve the power of Instance-fake. The objective is formalized as:
\begin{equation}
	\begin{gathered}
		\mathcal{C}_p^f=\mathcal{C}_{d(\cdot), \phi^f(\cdot)}\left(T_q(G(\mathbf{z}_{q})), T_{k^+}(G(\mathbf{z}_{q}+\epsilon_q)),\left\{T_{k_i}(G(\mathbf{z}_{k_{i}}))\right\}_{i=1}^{N}\right),
	\end{gathered}
	\label{eq:instance-perturbation}
\end{equation}
where $\epsilon_q$ stands for the perturbation of latent augmentation, which is sampled from a Gaussian distribution with small variance.

In the following, we adopt StyleGAN2-ADA \cite{karras2020training} and InsGen \cite{yang2021data} as reference and conduct ablation studies\footnote{ 
In InsGen \cite{yang2021data}, ADA-Linear rather than ADA is used as augmentation $T$  when the number of images in the dataset is less than 5K. 
For fairness, ADA-Linear, rather than ADA and InsGen, should be regraded as the baseline for comparison.} 
to explore \textit{"How Contrastive Learning Benefits Data-Efficient GANs?"}. Related results are reported in Table \ref{Tab:Insgen_ablation_1} and Fig. \ref{FIG:ablation of Insgen}.

\noindent\textbf{Instance-real causes performance drop on small datasets.} 
Although on large dataset such as FFHQ-5K, Instance-real\footnote{The objectives of discriminator are: i) Instance-real: $\mathcal{L}^{'}_{D}=\mathcal{L}_{D}+\lambda^r\mathcal{C}^r$;
ii) Instance-fake:$\mathcal{L}^{'}_{D}=\mathcal{L}_{D}+\lambda^f\mathcal{C}^f$;
iii) Instance-perturbation: $\mathcal{L}^{'}_{D}=\mathcal{L}_{D}+\lambda^f\mathcal{C}^f_p$, Respectively.\label{defination}} is effective, it would cause performance drop on smaller datasets. As shown in Fig. \ref{FIG:ablation of Insgen}, the performance with different queue size is consistently decreased on small datasets, \textit{e.g.} FFHQ-100, FFHQ-1K, and FFHQ-2K. Such a degradation is much severer when the queue size is 100 or 300 on FFHQ-100 dataset. 

\begin{table*}[t!]
	\caption{The whole ablation experiments of different contrastive learning strategies. We demonstrate the best FID metric for different queue sizes on various datasets. Results with $^*$ are best on four different queue sizes, which are better than those reported in \cite{yang2021data}.
		\label{Tab:Insgen_ablation_1}}
	\centering
	\begin{tabular}{ccccc}
		\toprule
		256$\times$256 Resolution$\quad$&$\quad$FFHQ-100$\quad$&$\quad$ FFHQ-1K$\quad$&$\quad$FFHQ-2K$\quad$&$\quad$FFHQ-5K\\
		\midrule
		StyleGAN2-ADA&85.8&21.29&15.39&10.96\\
		StyleGAN2-ADA-Linear& 82& 19.86& 13.01& 9.39\\
		Instance-real& 107& 20.68& 14.15& 8.50\\
		Instance-fake&76.6& 19.45& 12.69& 8.31\\
		Instance-perturbation&\textbf{43.35}& 17.46& \textbf{11.29}& 8.01\\
		\hline
		InsGen~\cite{yang2021data}&53.93&19.58&11.92&-\\
		$\text{InsGen}^*$&$45.75^*$& $18.21^*$& $11.47^*$&$\textbf{7.83}^*$\\
		\bottomrule
	\end{tabular}
\end{table*}

\begin{figure*}[t!]
\setlength{\abovecaptionskip}{0.1cm}
    \setlength{\belowcaptionskip}{-0.3cm}
	\centering
	\includegraphics[scale=0.35]{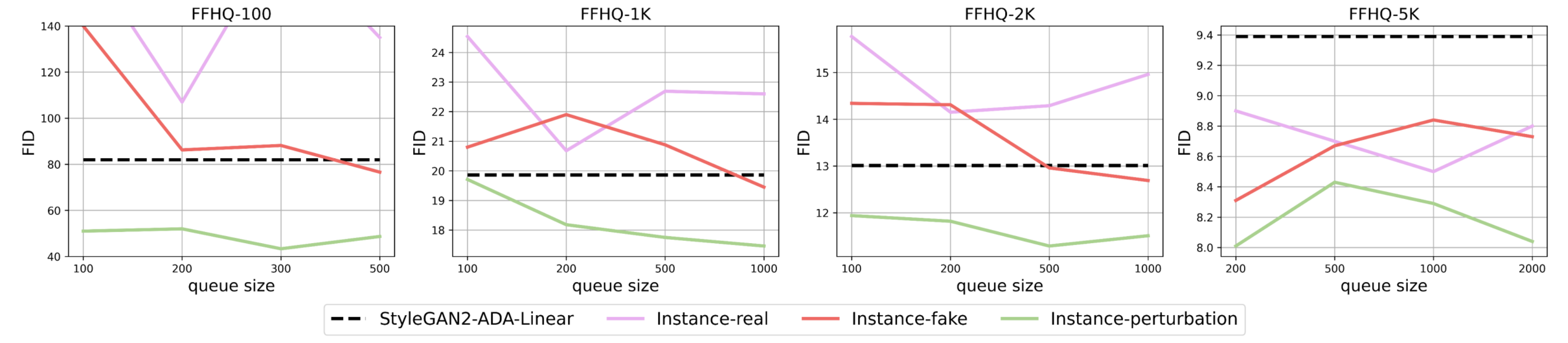}
	\caption{Part of the ablation experiments of different contrastive learning strategies with various queue sizes and datasets.}
	\label{FIG:ablation of Insgen}
\end{figure*}


\noindent\textbf{Instance-fake improves performance marginally on small datasets.} When adding Instance-fake\textsuperscript{\ref{defination}} strategy, compared to the baseline StyleGAN2-ADA-Linear, the best synthesis quality for different queue sizes is consistently improved in Table \ref{Tab:Insgen_ablation_1}, regardless of the number of training data. However, on small datasets such as FFHQ-100, FFHQ-1K, and FFHQ-2K, the increase is marginal, and the performance is even decreased when queue size is small (as shown in Fig. \ref{FIG:ablation of Insgen}). 

\noindent\textbf{Instance-perturbation is the key point for DE-GANs.} 
After using Instance-perturbation\textsuperscript{\ref{defination}}, FID obtains significant improvement regardless of queue size and the number of training images. Specifically, compared to Instance-fake, the best FID for different queue size as shown in Table \ref{Tab:Insgen_ablation_1} acquires consistent improvement. 
Furthermore, Instance-perturbation also obtains FID improvements of -2.4, -0.75, and -0.18 relative to InsGen (\textit{i.e.} employs instance-real at the same time) on FFHQ-100, FFHQ-1K, and FFHQ-2K datasets, respectively. 

\noindent\textbf{The optimal queue size is related to the diversity of negative samples.}
Following MoCo-v2 \cite{he2020momentum}, negative samples of InsGen are stored in a queue to reduce the computational complexity. Two queues, fake instance queue and real instance queue, are adopted in InsGen. The size of them are defined as \textit{fqs} and \textit{rqs} respectively. Empirically, length of the queue tends to be the $5\%$ number of the dataset in MoCo-v2, indicating \textit{rqs} is related to the number of real data. According to our experiments (Fig. \ref{FIG:ablation of Insgen}), we find that the optimal \textit{fqs} in Instance-fake are larger than the optimal \textit{rqs} in Instance-real under the same setting. We argue that the optimal queue size is related to the diversity of negative samples in queues. In Instance-real, the diversity of negative samples is low with limited training data. Thus long \textit{rqs} leads to the degradation of contrastive learning.

\subsection{Instance-perturbation and Latent Space Continuity} 

The above analyses demonstrate that Instance-perturbation is the key point of contrastive learning in DE-GANs and better performance than InsGen can be obtained by using Instance-perturbation alone. Moreover, we argue that Instance-perturbation mitigates the discontinuity of latent space effectively. Intuitively, Instance-perturbation introduces latent similarity prior that images synthesized from the latent codes within a neighbourhood are close to each other, which makes the latent space continuous. Experimentally, Fig. \ref{FIG:motivation} illustrates that Instance-perturbation owns more continuous interpolation, more accurate inversion, better FID, and better PPL compared to StyleGAN2-LDA, StyleGAN2-APA \cite{jiang2021deceive}, and other contrastive learning strategies on Obama dataset. 
Furthermore, Table \ref{Tab:Perceptual path length} shows that PPL \cite{karras2019style} of Instance-perturbation outperform other methods in Mean and Std. {We also show some similar results of FFHQ-100 dataset in Sec. A.1 of the supplementary materials.} 

\begin{table}[t!]
	\caption{Perceptual path length (PPL) of $z$ and $w$ spaces on FFHQ-2K (256$\times$256)
		\label{Tab:Perceptual path length}}
	\centering
	\begin{tabular}{ccccc}
		\toprule
		\multirow{2}{*}{Method}&\multicolumn{2}{c}{PPL ($z$)} &\multicolumn{2}{c}{PPL ($w$)} \\
		&$\quad$Mean$\quad$& $\quad$Std. Dev$\quad$&$\quad$Mean $\quad$&$\quad$Std. Dev\\
		\midrule
		StyleGAN2-ADA-Linear&2163&1819&232&100.3\\
		Instance-fake&2399&2082&217&95.6\\
		Instance-perturbation&1729&1626&167&90.5\\
		FakeCLR&\textbf{1411}&\textbf{1490}&\textbf{136}&\textbf{65.4}\\
		\bottomrule
	\end{tabular}
\end{table}

\section{Methodology}

As discussed above, real instance discrimination is not conducive to the DE-GANs' training. Hence, we only adopt Instance-perturbation, which is illustrated on the top part of Fig. \ref{FIG:FakeCLR}. Furthermore, bottom part of Fig. \ref{FIG:FakeCLR} introduces three innovative strategies, that are \textcolor[RGB]{0,176,240}{Noise-related Latent Augmentation}, \textcolor[RGB]{238,105,100}{Diversity-aware Queue}, and \textcolor[RGB]{169,209,142}{Forgetting Factor of Queue}, for improving the efficiency of contrastive learning on fake images. In summary, the proposed complete objective function\footnote{Followed InsGen \cite{yang2021data}, Instance-fake ($\mathcal{C}^f$) is added into the generator loss to improve the generated diversity.} of FakeCLR is optimized using:
\begin{equation}
	\begin{gathered}
		\mathcal{\hat L}_{D}=\mathcal{L}_{D}+\lambda^f\mathcal{\hat C}^f_p,\quad  \mathcal{\hat L}_{G}=\mathcal{L}_{G}+\lambda_G\mathcal{C}^f, \\
		\text{where}\quad \mathcal{\hat C}^f_p=\textcolor[RGB]{169,209,142}{\mathcal{\hat C}_{d(\cdot), \phi^f(\cdot)}}\left(T_q(G(\mathbf{z}_{q})), T_{k^+}(G(\mathbf{z}_{q}+\textcolor[RGB]{0,176,240}{\hat\epsilon_q})),\left\{T_{k_i}(G(\mathbf{z}_{k_{i}})),\mathbf{m}_i\right\}_{i=1}^{\textcolor[RGB]{238,105,100}{\hat N}}\right).
	\end{gathered}
\end{equation}

\noindent\textbf{Noise-related Latent Augmentation.}
From Sec. 3.2, noise perturbation is the key point for contrastive learning in DE-GANs, which alleviates the discontinuity of latent space effectively. From this perspective, noise perturbation introduces latent similarity prior to the discriminator, which provides latent-continuity guidance for the generator. However, the original strategy holds the same perturbation radius ($\epsilon_q$), which is not desirable due to the nonuniform latent distribution. More sensibly, a noise-related  perturbation with different latent sampling $z_q$ should be adopted. Generally, $\epsilon_q\propto \left|z_q\right|$ when $z_q\sim \mathcal{N}(0,1)$. It can be interpreted as stronger similarity prior should be required in low-density regions of the latent distribution. In this paper, we adopt a simple implementation that is \textcolor[RGB]{0,176,240}{$\hat{\epsilon}_q= l_1\cdot\left|z_q\right|$}, where $l_1$ is a coefficient. 

\begin{figure*}[t!]
\setlength{\abovecaptionskip}{0.1cm}
    \setlength{\belowcaptionskip}{-0.3cm}
	\centering
	\includegraphics[scale=0.6]{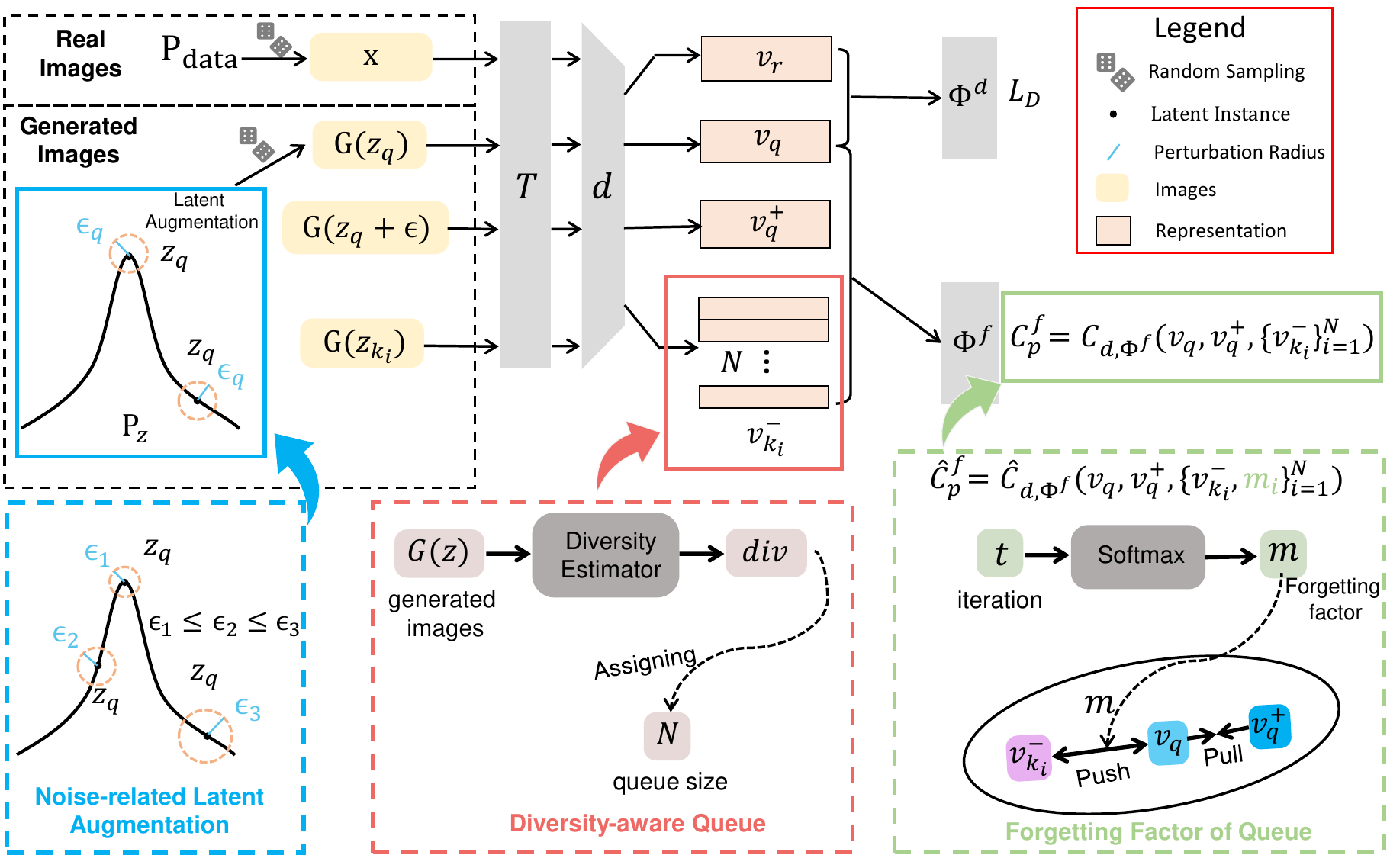}
	\caption{Overview of FakeCLR. Overall, the pipeline only adopting fake instance discrimination with latent augmentation is illustrated on the top of the figure, where the representation is learned from bi-classification loss ($\mathcal{L}_{D}$) and Instance-perturbation loss ($\mathcal{C}^f_{p}$). Here, $N$ is the number of negative samples (queue size), $T(\cdot)$ is the data augmentation, $d$ is the backbone of discriminator, $\phi^d$ and $\phi^f$ are two independent heads. Furthermore, three strategies proposed in this paper are illustrated at the bottom of the figure.}
	\label{FIG:FakeCLR}
\end{figure*}

\noindent\textbf{Diversity-aware Queue.}
The number of negative samples is critical in contrastive learning. Short queue leads to inhibiting convergence \cite{chen2020simple}, while long queue suffers from redundancy and inappropriate negative samples. Empirically, length of the queue is related to the size of dataset, such as $5\%$ number of the dataset in MoCo-v2~\cite{he2020momentum}. However, this empirical hypothesis does not seem to hold up in GANs due to the continuously changing generated images over training iteration. {As described in Sec. 3.1, the optimal queue size may be related to the diversity of samples in queue. Furthermore, the diversity of fake queues is indeed higher than that of real queues and decreases gradually accompanying with training due to the convergence and stability of generator.} Therefore, a diversity-aware queue is required for contrastive learning in fake images. Generally, $N\propto {DE}(G(z))$, where $DE(\cdot)$ is a diversity estimator. In this paper, we also adopt a simple implementation that is \textcolor[RGB]{238,105,100}{$\hat N=l_2\cdot t$} to reduce the computational complexity, where $l_2$ is a negative coefficient and $t$ is the iteration of training. Therefore, queue size in FakeCLR decreases linearly with respect to the training.

\noindent\textbf{Forgetting Factor of Queue.}
The choice of negative pairs has drawn much attention in contrastive learning. \cite{robinson2020contrastive,kalantidis2020hard} demonstrate that hard negative samples capture desirable generalization properties and improve the performance of downstream task. Different from real images, fake images are updated with training iteration. As a result, varying levels of generated samples exist in the negative queue.  
However, vanilla contrastive learning considers them to have the same importance. Some iteration-based prior, \textit{i.e.}, the negative samples generated by the current iteration should be given more attention, should be introduced. In this paper, we apply a forgetting factor to negative samples of the queue. Specifically, the contribution of previous samples should be smaller than current samples to contrastive loss. This makes the contrastive term focus more on recently generated samples. The iteration-based InfoNCE is formulated as:
\begin{equation}
	\begin{gathered}
{\mathcal{\hat C}_{F(\cdot), \phi(\cdot)}\left(\mathbf{x}_{q}, \mathbf{x}^+_k,\left\{\mathbf{x}^-_{k_{i}},\textcolor[RGB]{169,209,142}{\mathbf{m}_i}\right\}_{i=1}^{N}\right)=}\\{-\log \frac{\exp \left(\phi\left(\mathbf{v}_{q}\right)^{T} \phi\left(\mathbf{v}^+_{k}\right) / \tau\right)}{\exp \left(\phi\left(\mathbf{v}_{q}\right)^{T} \phi\left(\mathbf{v}^+_{k}\right) / \tau\right)+\sum_{i=1}^{N} \exp\left( \left(\phi\left(\mathbf{v}_{q}\right)^{T} \phi\left(\mathbf{v}^-_{k_{i}}\right)+\textcolor[RGB]{169,209,142}{\mathbf{m}_i}\right) / \tau\right)},}
	\end{gathered}
\end{equation}
where \textcolor[RGB]{169,209,142}{$\mathbf{m}_i$} is the introduced forgetting factor to adjust the importance of different negative samples, as defined by: 
\begin{equation}
	\begin{gathered}
		\quad {\mathbf{m}_i}=\frac{\exp\left(\mathbf{\hat t}_i/\tau_{\mathbf{m}}\right)}{\sum_{j=1}^{N}\exp\left(\mathbf{\hat t}_j/\tau_\mathbf{m}\right)},\quad \text{where}\quad \mathbf{\hat t}_i=\frac{\mathbf{t}_i-\min(\mathbf{t})}{\max(\mathbf{t})-\min(\mathbf{t})},\\ \mathbf{t}=\{\mathbf{t_1},\mathbf{t_2},\cdots,\mathbf{t}_N\}, \quad i=1,2,\cdots,N,
	\end{gathered}
\end{equation}
$\tau_\mathbf{m}$ is a temperature coefficient that controls the distribution $\mathbf{m} :=\{\mathbf{m_1}, \cdots,\mathbf{m}_N\}$, and $t_i$ means that the negative sample $\mathbf{v}^-_{k_{i}}$ is generated by the generator of the $t_i$-th iteration. The distribution of $\mathbf{m}$ under different temperature $\tau_\mathbf{m}$ and queue size $N$ has be shown in Sec. A.2 of supplementary materials. The proposed forgetting factor ($\mathbf{m}$) indeed improves the importance of current negative samples to contrastive term. The proof and discussion are in Sec. A.3 of supplementary materials. Furthermore, a PyTorch-like pseudocode of iteration-based contrastive learning has also be shown in Sec. A.4 of supplementary materials.

 \begin{table}[t!]
 \caption{Comparison with the state of the arts over FFHQ (256 $\times$ 256): Training with 100, 1K, 2K, and 5K samples. FID (lower is better) is reported as the evaluation metric. our method performs the best consistently. All methods adopt the StyleGAN2 architecture. Results with $^*$ are best on four different queue sizes, which are better than those reported in \cite{yang2021data}.
	\label{Tab:ffhq}}
\centering
\begin{tabular}{cccccc}
		\toprule
		Method&Augmentation&$\quad$FFHQ-100$\quad$&FFHQ-1K$\quad$&FFHQ-2k$\quad$&FFHQ-5K\\
		\midrule
		StyleGAN2~\cite{karras2020analyzing}&No&179&100.16&54.3&49.68\\
		GenCo~\cite{cui2021genco}&No&148&65.31&47.32 &27.96\\
		DISP \cite{mangla2022data}&Yes&-&-&21.06&-\\
		LeCam-GAN~\cite{tseng2021regularizing}&Yes&-&21.70&-&-\\
		ADA~\cite{karras2020training}&Yes&85.8&21.29&15.39&10.96\\
		ADA-Linear~\cite{yang2021data}&Yes& 82& 19.86& 13.01& 9.39\\
		APA~\cite{jiang2021deceive}&Yes&65&$18.89$&16.90&$8.38$\\
		InsGen~\cite{yang2021data}&Yes&53.93&19.58&11.92&-\\
		$\text{InsGen}^*$&Yes&$45.75^*$& $18.21^*$&$ 11.47^*$&$7.83^*$\\
		FakeCLR (ours)&Yes&\textbf{42.56}&\textbf{15.92}&\textbf{9.90}&\textbf{7.25}\\
		\bottomrule
	\end{tabular}
\end{table}

\section{Experiments}
We evaluate the advancement of our proposed FakeCLR on multiple datasets. After introducing implementation details, we present the comparison on FFHQ \cite{karras2019style} dataset. Sec. 5.3 presents the comparison on few-shot datasets (Obama~\cite{zhao2020differentiable}, Grumby Cat~\cite{zhao2020differentiable}, Panda~\cite{zhao2020differentiable}, and AnimalFace \cite{si2011learning}). Furthermore, ablation studies showing the importance of different components are demonstrated in Sec. 5.4.

\subsection{Implementation Details}
Similar to InsGen \cite{yang2021data}, our FakeCLR also can be easily implemented on any GAN framework. In this paper, we take the state-of-the-art GANs model, StyleGAN2 \cite{karras2020analyzing}, as an example to demonstrate how FakeCLR is implemented.

InsGen reuses the network structure and learning hyper-parameters of StyleGAN2 \cite{karras2020analyzing}, the augmentation pipeline of ADA \cite{karras2020training}, and contrastive learning settings of MoCo-v2 \cite{he2020momentum} achieving state-of-the-art results with limited training data. Additionally, a linear hyper-parameter that controls the strength of augmentation (ADA-Linear) has also been proposed to replace the adaptive hyper-parameter in ADA \cite{karras2020training}. Therefore, we exactly reuse the above settings in InsGen for a fair comparison. Besides, some new hyper-parameter in FakeCLR are setted as $l_1=0.1$ and $\tau_{\mathbf{m}}=0.01$. All experiments are conducted on two NVIDIA Tesla A100 GPUs with 64 batch size. 

\subsection{Results on FFHQ}
FFHQ contains 70K high-resolution images of human faces. To reduce the computational complexity, we resize images to 256$\times 256$. For the experiments of limited data, we collect a subset of training data by randomly sampling with 100, 1K, 2K, and 5K images. Regardless of the number of the training data, the Fréchet Inception Distance (FID) \cite{heusel2017gans} metric is calculated between 50K fake images and all 70K images, evaluating the performance in terms of both reality and diversity. Moreover, the implementation of our FakeCLR is based on the official implementation of \href{https://github.com/NVlabs/stylegan2-ada}{StyleGAN2-ADA} and \href{https://genforce.github.io/insgen/}{InsGen}. 
\begin{figure*}[t!]
\setlength{\abovecaptionskip}{0.1cm}
    \setlength{\belowcaptionskip}{-0.3cm}
	\centering
	\includegraphics[scale=0.42]{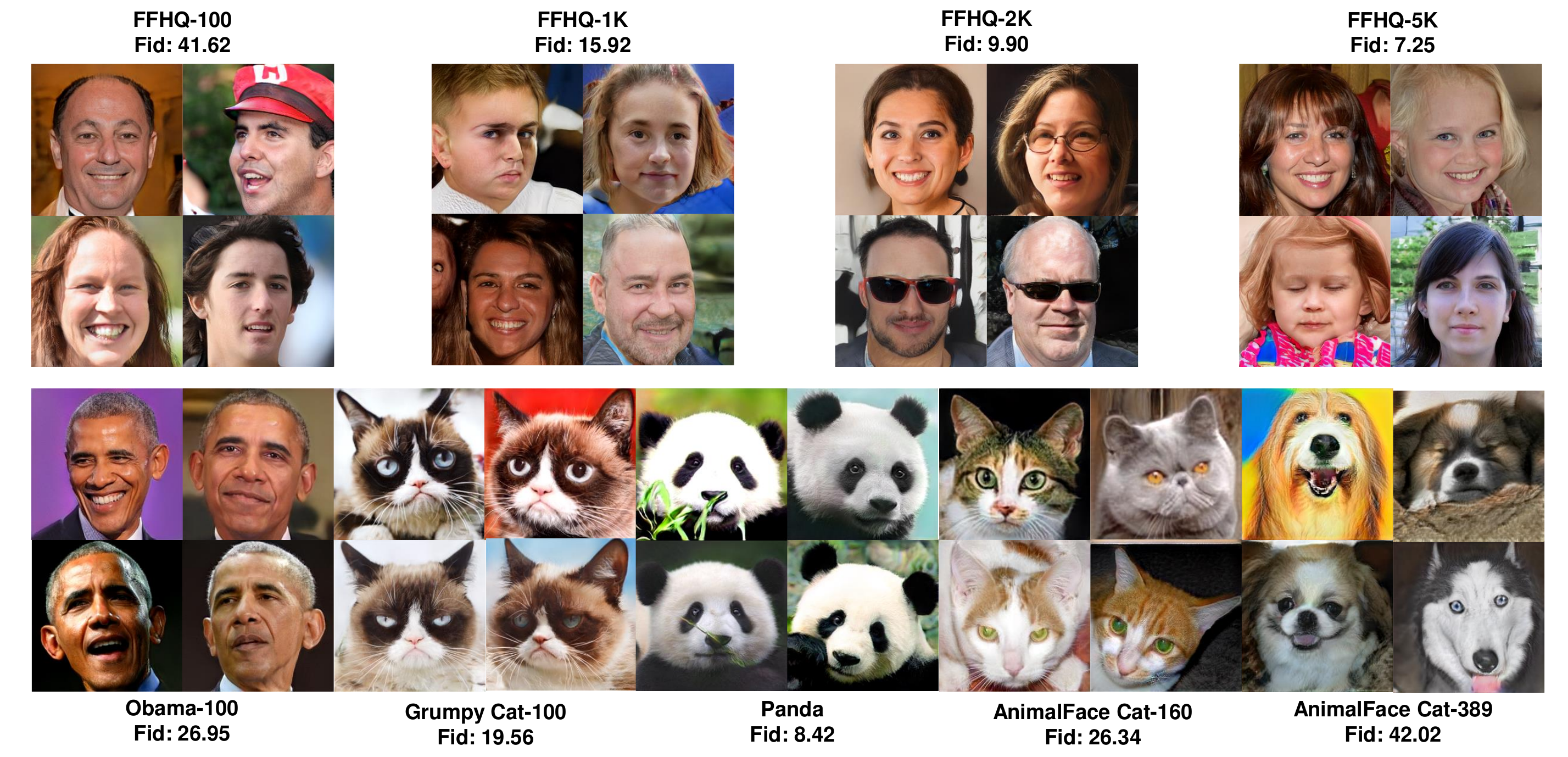}
	\caption{Generated images and the corresponding FID under various datasets. All images are synthesized randomly without truncation.}
	\label{FIG:generated_images}
\end{figure*}

Table \ref{Tab:ffhq} shows the comparison with previous studies on FFHQ-100, FFHQ-1K, FFHQ-2K, and FFHQ-5k. We compare to most recent works, and our method achieves the new state-of-the-art consistently. Specifically, compared to the baseline method InsGen, our method improves the FID on FFHQ by 3.19, 2.29, 1.57, and 0.58 with 100, 1K, 2K, and 5K training images, respectively. Such results show the effectiveness of our proposed FakeCLR, especially for data-efficient scenarios. Some generated examples and the corresponding FID on FFHQ are presented on the top part of Fig. \ref{FIG:generated_images}. More examples generated by InsGen \cite{yang2021data} and our FakeCLR are presented in Sec. A.5 of the supplementary materials. Furthermore, we also analyze the computation cost in Sec. A.6 of the supplementary materials. Note that the three strategies adopted in FakeCLR only introduce a fraction of the cost. More importantly, our FakeCLR achieves better results even with $8.8\%$ less computational cost than InsGen \cite{yang2021data}. {Some qualitative analysis of nearest neighbor test as done in DiffAugment \cite{zhao2020differentiable} in the LPIPS and pixel space, and comparisons on PPL and KID metrics on the FFHQ-100 dataset have also been presented in Sec. A.7 and Sec. A.8 of the supplementary materials, respectively.}

\subsection{Results on Few-shot Generation}
We also evaluate the advancement of FakeCLR on few-shot generation that contains five datasets (\textit{i.e.}, Obama, Grumpy Cat, Panada, AnimalFace Cat, and AnimalFace Dog) with 100, 100, 100, 160, and 389 training images, respectively. All models in this section are trained with the resolution of $256 \times 256$, and the datasets can be found at the \href{https://drive.google.com/file/d/1aAJCZbXNHyraJ6Mi13dSbe7pTyfPXha0/view}{link}. FID \cite{heusel2017gans} is calculated between 50K fake images and the training images. Other settings are the same as experiments on FFHQ. 

Table \ref{Tab:few-shot} demonstrates the quantitative comparison on few-shot generation. Top five baselines in Table \ref{Tab:few-shot}  pre-train the model with large FFHQ-70K. It can be observed that FakeCLR can even achieve better FID by using only 100-400 training samples. The synthesis quality is substantially improved by our method on most datasets except for Grumpy Cat datasets. Specifically, we obtain an FID improvements over GenCo \cite{cui2021genco} by $16.3\%$, $-9.9\%$, $11.3\%$, $14.7\%$, and $15.3\%$ on five datasets, respectively. Compared to InsGen \cite{yang2021data}, the FID improvements are $16.9\%$, $11.1\%$, $14.5\%$, $20.2\%$, and $15.3\%$, respectively. Moreover, some qualitative results are shown on the second row of Fig. \ref{FIG:generated_images}. More examples generated by InsGen \cite{yang2021data} and FakeCLR can be found in Sec. A.9 of the supplementary. {Furthermore, nearest neighbor test in the LPIPS and pixel space, and comparisons on PPL and KID metrics on the Obama dataset have also been presented in Sec. A.7 and Sec. A.8 of the supplementary, respectively.}

\begin{table}[t!]
\scriptsize
  \caption{Comparison with previous works over 100-shot and AFHQ: Training with 100 (Obama, Grumpy Cat, and Panda), 160 (AFHQ Cat), and 389 (AFHQ Dog) samples. FID (lower is better) is reported as the evaluation metric. Our method performs the best in most datasets. All methods adopt the StyleGAN2 architecture except for FastGAN+DA.
		\label{Tab:few-shot}}
\centering
\begin{tabular}{c|c|c|ccc|cc}
\hline 	\multirow{2}{*}{Method} &{Massive} & Pre-training & \multicolumn{3}{|c|}{ 100-shot } & \multicolumn{2}{c}{ AnimalFace } \\
& Augmentation & w/ 70K images & Obama & Grumpy Cat & Panda & Cat & Dog \\
\hline Scale/shift \cite{noguchi2019image} & No & Yes & 50.72 & 34.20 & 21.38 & 54.83 & 83.04 \\
MineGAN \cite{wang2020minegan} & No & Yes & 50.63 & 34.54 & 14.84 & 54.45 & 93.03 \\
TransferGAN \cite{wang2018transferring} & No & Yes & 48.73 & 34.06 & 23.20 & 52.61 & 82.38 \\
TransferGAN + DA \cite{zhao2020differentiable} & Yes & Yes & 39.85 & 29.77 & 17.12 & 49.10 & 65.57 \\
FreezeD \cite{mo2020freeze} & No & Yes & 41.87 & 31.22 & 17.95 & 47.70 & 70.46 \\
\hline StyleGAN2 \cite{karras2020analyzing} & No & No & 80.20 & 48.90 & 34.27 & 71.71 & 130.19 \\

\hline DA \cite{zhao2020differentiable} & Yes & No & 46.87 & 27.08 & 12.06 & 42.44 & 58.85 \\
FastGAN+DA \cite{liu2020towards}&Yes&No&41.05&26.65&10.03&35.11&50.66\\
ADA \cite{karras2020training} & Yes & No & 45.69 & 26.62 & 12.90 & 40.77 & 56.83 \\
ADA-Linear \cite{yang2021data}& Yes & No &34.41&25.67&10.23&36.30&54.39\\
SGLP \cite{kong2021smoothing}& Yes & No & 45.37 & 26.52 & -& - & - \\
LeCam-GAN \cite{tseng2021regularizing} & Yes & No & 33.16 & 24.93 & 10.16 & 34.18 & 54.88 \\
GenCo \cite{cui2021genco}& Yes & No & 32.21&\textbf{17.79}&9.49& 30.89& 49.63 \\
InsGen~\cite{yang2021data}&Yes&No&32.42&22.01&9.85&33.01&44.93\\
FakeCLR (Ours)&Yes&No&\textbf{26.95}&19.56&\textbf{8.42}&\textbf{26.34}&\textbf{42.02}\\
\hline
\end{tabular}
\end{table}

\subsection{Ablation Study}
\begin{table*}[t!]
	\caption{Ablation study of FakeCLR. The proposed strategies in FakeCLR all improve the generation over the baseline. Here, the first line is baseline StyleGAN2-ADA-Linear, and the other lines mean adding the proposed strategies.
		\label{Tab:abalation study}}
	\centering
	\begin{tabular}{cccc|cccc}
		\toprule
		$\mathcal{C}^f_p$\quad&\quad\textcolor[RGB]{0,176,240}{$\hat{\epsilon}_q$}\quad&\quad$\textcolor[RGB]{169,209,142}{\mathcal{\hat C}_{d(\cdot), \phi^f(\cdot)}}$\quad&\quad$\textcolor[RGB]{238,105,100}{\hat N}\quad$&\quad FFHQ-100\quad&\quad FFHQ-1K\quad&\quad FFHQ-2K\quad&\quad FFHQ-5K\\
		\midrule
		&&&&\quad82.00&\quad 19.86&\quad 13.01&\quad 9.39\\
		\checkmark&&&&\quad43.35&\quad 17.46&\quad 11.29&\quad 8.01\\
		\checkmark\quad&\quad\checkmark&&&\quad44.10&\quad17.09&\quad10.65&\quad7.89\\
		\checkmark\quad&\quad\checkmark&\checkmark&&\quad\textbf{41.62}&\quad16.05&\quad10.40&\quad7.48\\
		\checkmark\quad&\quad\checkmark&\checkmark&\checkmark&\quad42.56&\quad\textbf{15.92}&\quad\textbf{9.90}&\quad\textbf{7.25}\\
		
		\bottomrule
	\end{tabular}
\end{table*}
\textbf{Component-level investigation}. We investigate the effectiveness of each component on the FFHQ dataset. As shown in Table \ref{Tab:abalation study}, all components are effective and can improve the performance 
Among them, $\textcolor[RGB]{169,209,142}{\mathcal{\hat C}_{d(\cdot), \phi^f(\cdot)}}$ achieves the most significant improvement. In addition, combining all components performs the best on FFHQ-1K, FFHQ-2K, and FFHQ-5K. On FFHQ-100, Diversity-aware Queue does not further improve the performance because the training data size is extremely small. Detailed ablations on FFHQ-2K are introduced as follows and more results can be found in Sec. A.10, A.11, and A.12 of the supplementary.

\noindent\textbf{Noise-related Latent Augmentation}. Perturbation Radius ($\epsilon_q$) is an important hyper-parameter, and we compare different $\epsilon_q$ to our Noise-related latent augmentation ($\hat\epsilon_q$) in Fig. \ref{FIG:ablation_others} (a). The result shows that our strategy outperforms the best fixed strategy with a clear margin. Furthermore, negative prior that is $\epsilon_q\propto 1/\left|z_q\right|$ has been introduced in Negative Noise-related Latent Augmentation. The results demonstrate that negative prior severely compromise performance, which further supports our motivation of Noise-related Latent Augmentation.

\noindent\textbf{Forgetting Factor of Queue}. The queue of negative samples is always a key factor in contrastive learning. An effective negative queue is supposed to be with diverse and up-to-date generated samples. Simply tuning batch size would cause a trade-off, \textit{i.e.}, small batch size involves old samples, and large batch size reduces data diversity. As shown in Fig. \ref{FIG:ablation_others} (b), 256 is the optimal batch size. And proposed Forgetting Factor of Queue provides a solution to break such trade-off, which suppresses old samples involved by small batch size and can achieve better performance. Note that for fair comparisons, we keep the training batch size constant and only adopt gradient accumulation to change the number of negative samples entering the queue for each iteration in contrastive learning (\textit{i.e.}, batch size in Forgetting Factor of Queue). Fig. \ref{FIG:ablation_others} (c) further shows the impact of Temperature ($\tau_\mathbf{m}$) in our method. As shown in Sec. A.2 of the supplementary materials, $\tau_\mathbf{m}$ controls the distribution of $\mathbf{m}$ and smaller $\tau_\mathbf{m}$ indicate higher weights for samples in current iteration. The results show that proposed method outperform the original method, 
and $\tau_\mathbf{m}=0.001$ achieves the best performance.

\begin{figure*}[t!]
\setlength{\abovecaptionskip}{0.1cm}
    \setlength{\belowcaptionskip}{-0.3cm}
	\centering
	\includegraphics[scale=0.58]{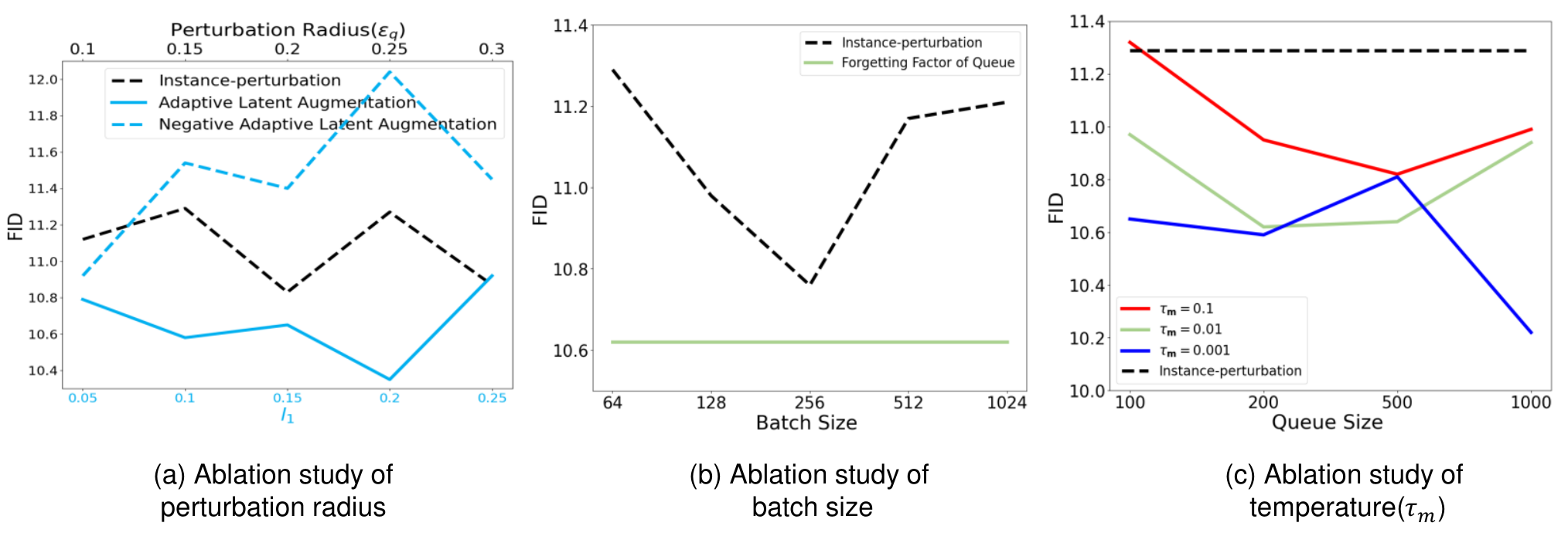}
	\caption{Ablation studies of different components in strategies of FakeCLR. (a): Ablation study of perturbation radius ($\epsilon_q$) in Noise-related Latent Augmentation. (b): Ablation study of batch size in Forgetting Factor of Queue. (c): Ablation study of temperature ($\tau_{\mathbf{m}}$) in Forgetting Factor of Queue.}
	\label{FIG:ablation_others}
\end{figure*}

\section{Conclusions}
 In this paper, we revisit the principle of different contrastive learning strategies in DE-GANs. According to experiments, we identify that latent similarity prior introduced by Instance-perturbation makes the main contribution to DE-GANs’ performance. Furthermore, we also explore and emphasize a major bottleneck of existing DE-GAN methods, discontinuous of the latent space and demonstrate that Instance-perturbation can mitigate it effectively. Based on these, we propose a new contrastive learning method, FakeCLR, in DE-GANs, 
 which acquires a new state of the art on both few-shot generation and limited-data generation.

\paragraph{\textbf{Acknowledgements}}
The work is partially supported by the National Natural Science Foundation of China under Grand No.U19B2044 and No.61836011.

\clearpage
\bibliographystyle{splncs}
\bibliography{egbib}

\appendix
\onecolumn
\section{Supplementary Materials}
\subsection{Qualitative Results of latent space continuity on FFHQ-100 dataset}
Some visual interpolation results on FFHQ-2K dataset are shown in Fig. \ref{FIG:motivation_ffhq-2k}.

\begin{figure*}[t!]
	\centering
	\includegraphics[scale=0.45]{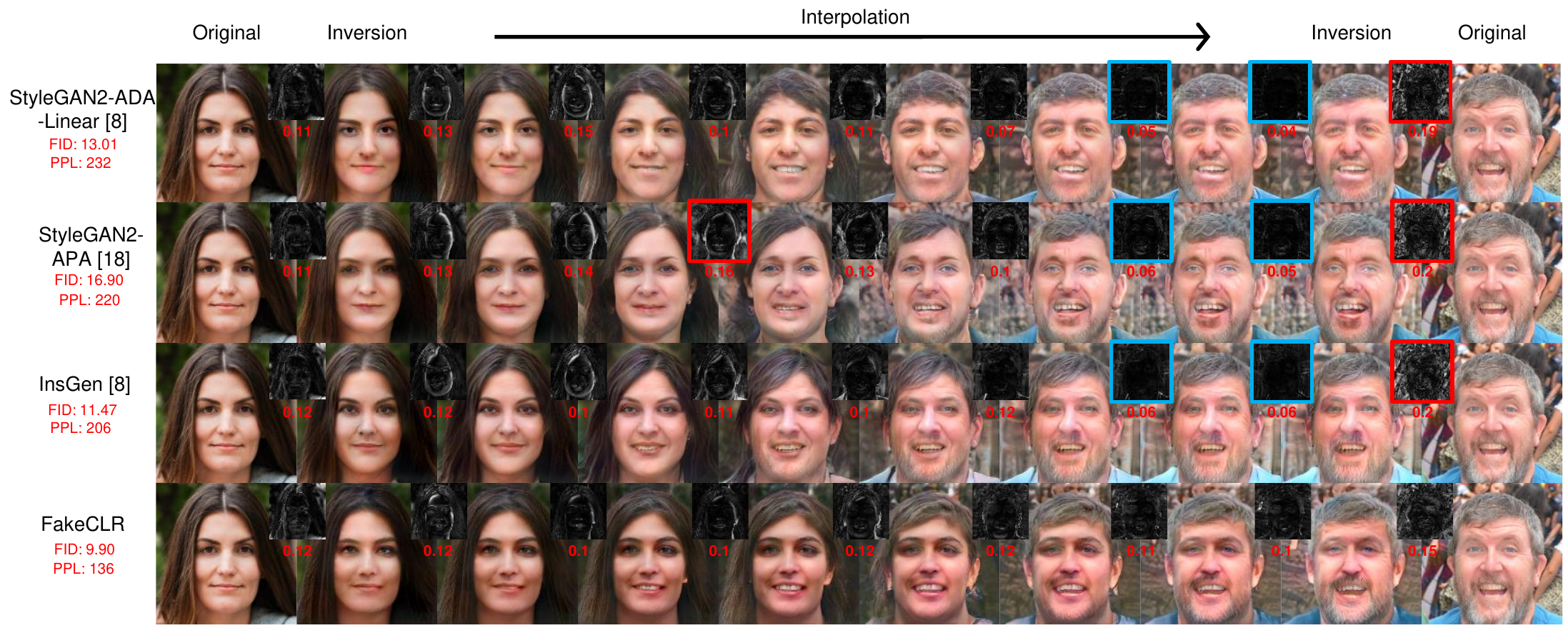}
	\caption{Training GANs with 2K training samples (FFHQ-2K dataset) typically results in severe discontinuity in latent space, \textit{i.e.} under-diversity interpolation on the top three rows. Compared to some reference studies, generator trained with FakeCLR show more accurate inversion, smoother latent space, more diverse interpolation, and better FID and PPL. The small grey images visualize the difference between the two face images. The red numbers are mean of pixel-wise difference, we highlighted the difference score$>0.15$ with \cy{ red} border and the difference score$\leq0.6$ with \hl{ blue} border}
	\label{FIG:motivation_ffhq-2k}
\end{figure*}
\subsection{The Distribution of Forgetting Factor ($\mathbf{m}$)}
Fig. \ref{figure:weight_distribution} illustrates the distribution of forgetting factor $\mathbf{m}$ under different temperature coefficient $\tau_\mathbf{m}$ and queue size $N$.
\begin{figure}[t!]
	\begin{subfigure}[b]{0.23\textwidth}
		\centering
		\includegraphics[width=\textwidth]{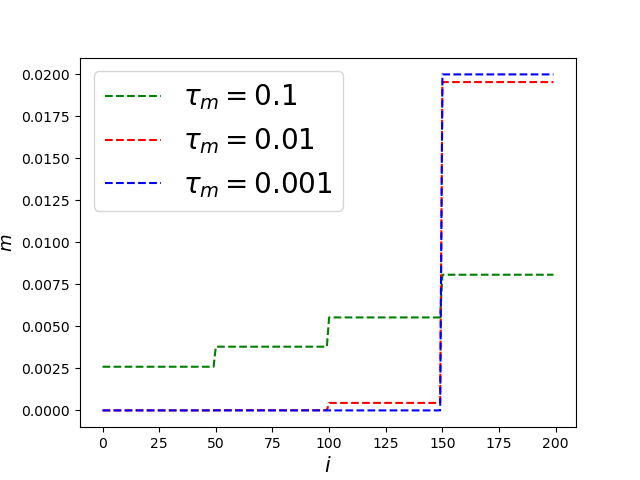}
		\caption{N=200}
		\label{figure:adv attack without adv training}
	\end{subfigure}
	\begin{subfigure}[b]{0.23\textwidth}
		\centering
		\includegraphics[width=\textwidth]{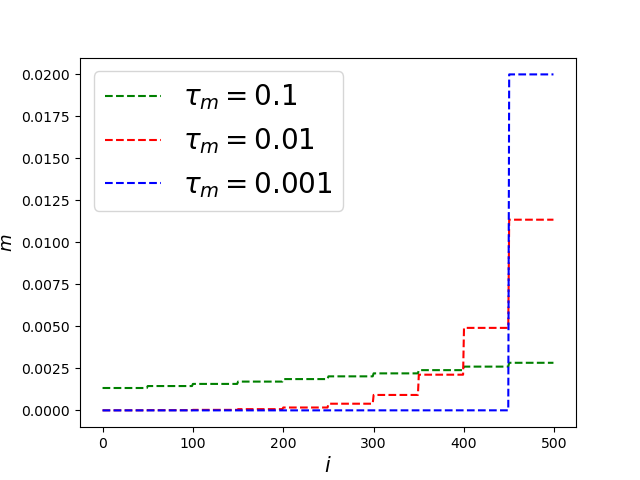}
		\caption{N=500}
		\label{figure:adv attack with adv training}
	\end{subfigure}
	\begin{subfigure}[b]{0.23\textwidth}
		\centering
		\includegraphics[width=\textwidth]{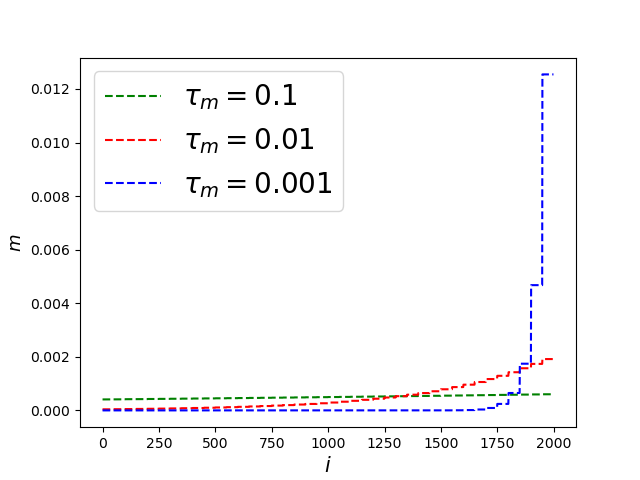}
		\caption{N=2000}
		\label{figure:adv attack with adv training}
	\end{subfigure}
	\begin{subfigure}[b]{0.23\textwidth}
		\centering
		\includegraphics[width=\textwidth]{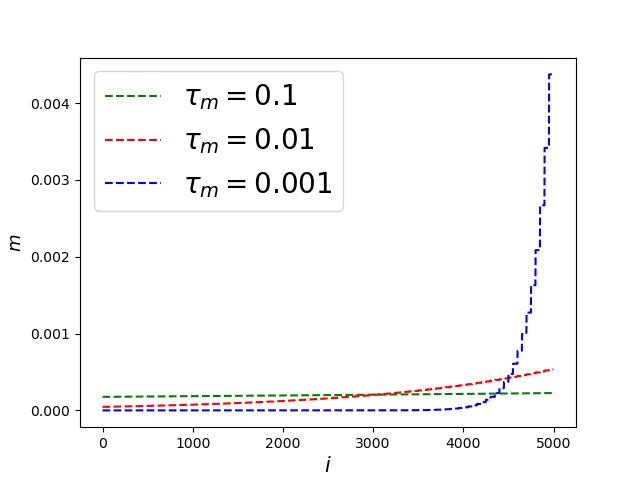}
		\caption{N=5000}
		\label{figure:adv attack without adv training}
	\end{subfigure}
	\centering
	\caption{Weight distribution $\mathbf{m}$ under different temperature coefficient $\tau_\mathbf{m}$ and queue size $N$, where the batch size is set to 50 (for illustrative purposes only).}
	\label{figure:weight_distribution}
\end{figure}

\subsection{Proof and Discussion of Iteration-based contrastive learning}

Adding the iteration-based weight ($\color{red}{\mathbf{m}_i}$) to negative samples leads to a iteration-based InfoNCE loss:
\begin{equation}
	\begin{gathered}
{\mathcal{\hat C}_{F(\cdot), \phi(\cdot)}\left(\mathbf{x}_{q}, \mathbf{x}^+_k,\left\{\mathbf{x}^-_{k_{i}},{\color{red}\mathbf{m}_i}\right\}_{i=1}^{N}\right)=}\\{-\log \frac{\exp \left(\phi\left(\mathbf{v}_{q}\right)^{T} \phi\left(\mathbf{v}^+_{k}\right) / \tau\right)}{\exp \left(\phi\left(\mathbf{v}_{q}\right)^{T} \phi\left(\mathbf{v}^+_{k}\right) / \tau\right)+\sum_{i=1}^{N} \exp\left( \left(\phi\left(\mathbf{v}_{q}\right)^{T} \phi\left(\mathbf{v}^-_{k_{i}}\right)+{\color{red}\mathbf{m}_i}\right) / \tau\right)},}
	\end{gathered}
\end{equation}

It may be useful to let: 
\begin{equation}
	\begin{gathered}
Y=\exp \left(\phi\left(\mathbf{v}_{q}\right)^{T} \phi\left(\mathbf{v}^+_{k}\right) / \tau\right)+\sum_{i=1}^{N} \exp \left(\left(\phi\left(\mathbf{v}_{q}\right)^{T} \phi\left(\mathbf{v}^-_{k_{i}}\right)+{\color{red}\mathbf{m}_i}\right) / \tau\right), \\
U=\exp \left(\phi\left(\mathbf{v}_{q}\right)^{T} \phi\left(\mathbf{v}^+_{k}\right) / \tau\right).
	\end{gathered}
\end{equation}

\begin{proposition}

The gradient of iteration-based InfoNCE loss is $\nabla\mathcal{\hat C}_{F(\cdot), \phi(\cdot)}=\frac{\nabla Y\cdot U-Y\cdot \nabla U}{Y\cdot U}$. Specifically, 
\begin{equation}
	\left\{
	\begin{aligned}
		\nabla_{\phi\left(\mathbf{v}_{q}\right)}\mathcal{\hat C}_{F(\cdot), \phi(\cdot)}&=\frac{\sum_{i=1}^{N} \exp \left(\left(\phi\left(\mathbf{v}_{q}\right)^{T} \phi\left(\mathbf{v}^-_{k_{i}}\right)+{\color{red}\mathbf{m}_i}\right) / \tau\right)\cdot\left(\phi\left(\mathbf{v}^-_{k_{i}}\right)-\phi\left(\mathbf{v}^+_{k}\right)\right)}{Y\cdot\tau}\\
		\nabla_{\phi(\mathbf{v}^+_{k})}\mathcal{\hat C}_{F(\cdot), \phi(\cdot)}&=-\frac{\sum_{i=1}^{N} \exp \left(\left(\phi\left(\mathbf{v}_{q}\right)^{T} \phi\left(\mathbf{v}^-_{k_{i}}\right)+{\color{red}\mathbf{m}_i}\right) / \tau\right)\cdot\phi\left(\mathbf{v}_q\right)}{Y\cdot\tau}\\
		\nabla_{\phi(\mathbf{v}^-_{k_{i}})}\mathcal{\hat C}_{F(\cdot), \phi(\cdot)}&=\frac{\exp \left(\left(\phi\left(\mathbf{v}_{q}\right)^{T} \phi\left(\mathbf{v}^-_{k_{i}}\right)+{\color{red}\mathbf{m}_i}\right)/ \tau\right)\cdot\phi\left(\mathbf{v}_{q}\right)}{Y\cdot\tau}
	\end{aligned}
\label{eq:loss_gradient}
\right.
\end{equation}
\end{proposition}

\begin{proof}
\begin{equation}
\scriptsize
	\begin{aligned}
		&\nabla_{\phi(\mathbf{v}_{q})}\mathcal{\hat C}_{F(\cdot), \phi(\cdot)}=\frac{\nabla_{\phi(\mathbf{v}_{q})}Y}{Y}-\frac{\nabla_{\phi(\mathbf{v}_{q})}U}{U}\\&=\frac{\exp \left(\phi\left(\mathbf{v}_{q}\right)^{T} \phi(\mathbf{v}^+_{k}) / \tau\right)\cdot \phi(\mathbf{v}^+_{k})+\sum_{i=1}^{N} \exp \left(\left(\phi\left(\mathbf{v}_{q}\right)^{T} \phi(\mathbf{v}^-_{k_{i}})+{\color{red}\mathbf{m}_i}\right) / \tau\right)\cdot\phi(\mathbf{v}^-_{k_{i}})-Y\cdot\phi(\mathbf{v}^+_{k}) }{Y\cdot\tau}\\
		&=\frac{\sum_{i=1}^{N} \exp \left(\left(\phi\left(\mathbf{v}_{q}\right)^{T} \phi(\mathbf{v}^-_{k_{i}})+{\color{red}\mathbf{m}_i}\right) / \tau\right)\cdot\left(\phi(\mathbf{v}^-_{k_{i}})-\phi(\mathbf{v}^+_{k})\right)}{Y\cdot\tau}\\
		&\nabla_{\phi(\mathbf{v}^+_{k})}\mathcal{\hat C}_{F(\cdot), \phi(\cdot)}=\frac{\nabla_{\phi(\mathbf{v}^+_{k})}Y}{Y}-\frac{\nabla_{\phi(\mathbf{v}^+_{k})}U}{U}
		=\frac{U\cdot\phi\left(\mathbf{v}_{q}\right)}{Y\cdot\tau}-\frac{\phi\left(\mathbf{v}_{q}\right)}{\tau}\\
		&=-\frac{\sum_{i=1}^{N} \exp \left(\left(\phi\left(\mathbf{v}_{q}\right)^{T} \phi(\mathbf{v}^-_{k_{i}})+{\color{red}\mathbf{m}_i}\right) / \tau\right)\cdot\phi\left(\mathbf{v}_q\right)}{Y\cdot\tau}\\
		&\nabla_{\phi(\mathbf{v}^-_{k_{i}})}\mathcal{\hat C}_{F(\cdot), \phi(\cdot)}=\frac{\nabla_{\phi(\mathbf{v}^-_{k_{i}})}Y}{Y}-\frac{\nabla_{\phi(\mathbf{v}^-_{k_{i}})}U}{U}
		=\frac{\exp \left(\left(\phi\left(\mathbf{v}_{q}\right)^{T} \phi(\mathbf{v}^-_{k_{i}})+{\color{red}\mathbf{m}_i}\right)/ \tau\right)\cdot\phi\left(\mathbf{v}_{q}\right)}{Y\cdot\tau}
	\end{aligned}
\end{equation}
\end{proof}

As we can see, all of the gradient of InfoNCE loss in Eq. (\ref{eq:loss_gradient}) are related to iteration-based weight ${\color{red}\mathbf{m}_i}$. Eq. (\ref{eq:loss_gradient}) makes intuitive sense: (i) The proposed ${\color{red}\mathbf{m}_i}$ improves the gradient from the positive samples. (ii) In particular, for negative samples, the proposed ${\color{red}\mathbf{m}_i}$ improves the gradient of end-of-queue samples that are more similar to the query and can be considered as hard samples, which significantly improves the efficiency of contrastive learning.

\subsection{Pseudocode of Iteration-based Contrastive Learning}
We also demonstrate the PyTorch-like pseudocode in Algorithm \ref{algo:time-based}.

\definecolor{commentcolor}{RGB}{110,154,155}   
\newcommand{\PyComment}[1]{\ttfamily\textcolor{commentcolor}{\# #1}}  
\newcommand{\PyCode}[1]{\ttfamily\textcolor{black}{#1}} 
\begin{algorithm}[h]
    \scriptsize
	\SetAlgoLined
	\PyComment{f\_q, f\_k: encoder networks for query and key (NxC)} \\
	\PyComment{queue: dictionary as a queue of K keys (CxK)} \\
	\PyComment{queue\_label: dictionary as a queue of K iteration-based label (1xK)} \\
	\PyComment{m: momentum} \\
	\PyComment{t: temperature of contrastive learning} \\
	\PyComment{t\_weight: temperature of iteration-based weight (0.01)} \\
	\PyCode{f\_k.params = f\_q.params} \PyComment{initialize}\\
	\PyCode{for x in loader:} \PyComment{load a minibatch x with N samples}\\
	\Indp   
		\PyCode{q = aug(x)} \PyComment{queries: NxC}\\
		\PyCode{k = aug(x)} \PyComment{keys: NxC}\\
		\PyCode{k = k.detach()} \PyComment{no gradient to keys}\\
		\PyCode{} \\
		\PyComment{positive logits: Nx1}\\
		\PyCode{l\_pos = torch.einsum('nc,nc->n', [q, k]).unsqueeze(-1)}\\
		\PyComment{negative logits: NxK}\\
		\PyCode{l\_neg = torch.einsum('nc,ck->nk', [q, queue.clone().detach()])}\\
		\PyCode{} \\
		\PyComment{iteration-based weight for negative samples}\\
		\PyCode{iteration\_weight=queue\_label.clone().detach()}\\
		\PyCode{iteration\_weight=(iteration\_weight-min(iteration\_weight))\\/(max(iteration\_weight)-min(iteration\_weight))}\\
		\PyCode{iteration\_weight=torch.normalize(iteration\_weight,p=2,dim=0)}\PyComment{normalize}\\
		\PyCode{iteration\_weight=F.softmax(iteration\_weight/t\_weight,dim=0)}\\
		\PyCode{} \\
		\PyComment{iteration-weighted negative logits}\\
		\PyCode{l\_neg=l\_neg+iteration\_weight}\\
		\PyCode{} \\
		\PyComment{logits: Nx(1+K)}\\
		\PyCode{logits = cat([l\_pos, l\_neg], dim=1)}\\
		\PyCode{} \\
		\PyComment{contrastive loss}\\
		\PyCode{labels = zeros(N) }\PyComment{positives are the 0-th}\\
		\PyCode{loss = CrossEntropyLoss(logits/t, labels)} \\
		\PyCode{} \\
		\PyComment{Adam update: query network}\\
		\PyCode{loss.backward()}\\
		\PyCode{update(f\_q.params)}\\
		\PyCode{} \\
		\PyComment{momentum update: key network}\\
		\PyCode{f\_k.params = m*f\_k.params+(1-m)*f\_q.params}\\
		\PyCode{} \\
		\PyComment{update dictionary}\\
		\PyCode{enqueue(queue,queue\_label,k)} \PyComment{enqueue the current minibatch}\\
		\PyCode{dequeue(queue,queue\_label)}\PyComment{dequeue the earliest minibatch}\\
	
	\caption{Pseudocode of iteration-based contrastive learning in a PyTorch-like style.}
	\label{algo:time-based}
\end{algorithm}
\subsection{Generated Images on FFHQ}
Fig. \ref{FIG:generated_imgaes_ffhq} illustrates some randomly synthesized images on FFHQ dataset.
\begin{figure*}
\setlength{\abovecaptionskip}{0.1cm}
    \setlength{\belowcaptionskip}{-0.3cm}
	\centering
	\includegraphics[scale=0.28]{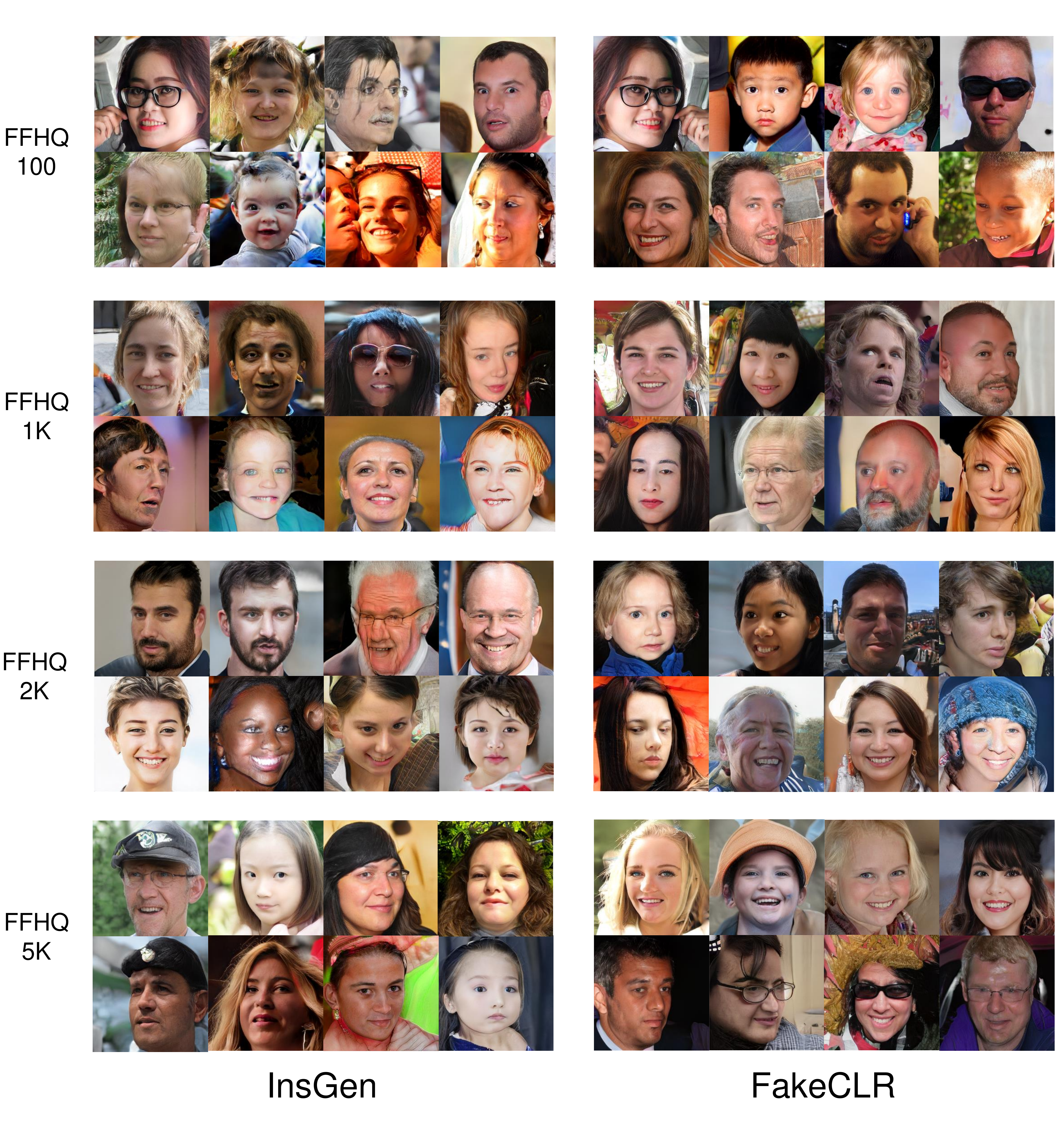}
	\caption{Generated images with different number of training images on FFHQ dataset. All images are synthesized randomly without truncation.}
	\label{FIG:generated_imgaes_ffhq}
\end{figure*}

\subsection{Analysis of Computation Cost}
Although three strategies have been adopted in our proposed FakeCLR, they are all lightweight and only a fraction of the cost has been introduced. The cost of time on FFHQ (256$\times$256) datasets have been demonstrated in Fig. \ref{Tab:computation cost}. 

\begin{table*}
	\caption{The cost of time on FFHQ (256$\times$256) datasets. All results are calculated by averaged over 10 times on two NVIDIA Tesla A100 GPUs with 64 batch size.
		\label{Tab:computation cost}}
	\centering
	\begin{tabular}{cc}
		\toprule
		256$\times$256 Resolution$\quad$& sec/kimg\\
		\midrule
		StyleGAN2-ADA&6.16\\
		StyleGAN2-ADA-Linear&6.15\\
		Instance-real&8.18\\
		Instance-fake&10.06\\
		InsGen&11.70\\
		FakeCLR&10.67\\
		\bottomrule
	\end{tabular}
\end{table*}

\subsection{Nearest Neighbor Test on FFHQ-100 and Obama Datasets}
We show the nearest neighbor test in pixel and LPIPS spaces on FFHQ-2K and Obama datasets in Fig. \ref{fig:nearest neighbors_obama} and Fig. \ref{fig:nearest neighbors_FFHQ-100}, respectively. The results indicates that our FakeCLR has more diverse generated images and does not simply memorize the training images even give small datasets.

\begin{figure}[t]
	\begin{center}
		\includegraphics[width=1\linewidth]{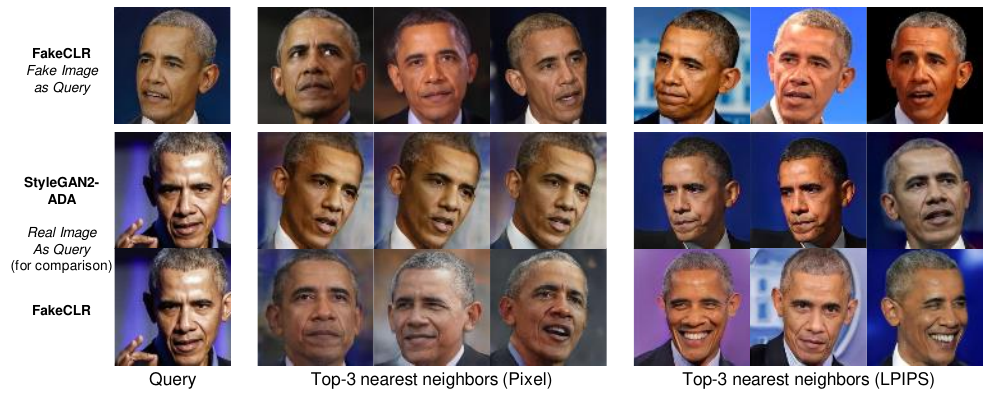}
	\end{center}
	\caption{Nearest neighbors in pixel space (left) and LPIPS feature space (right) on the Obama dataset. Followed as DiffAugment \cite{zhao2020differentiable}, we select fake image as query in the first row. Furthermore, we select the real image as new query for comparing with previous methods in the bottom two rows.}
	\label{fig:nearest neighbors_obama}
\end{figure}
\begin{figure}[t]
	\begin{center}
		\includegraphics[width=1\linewidth]{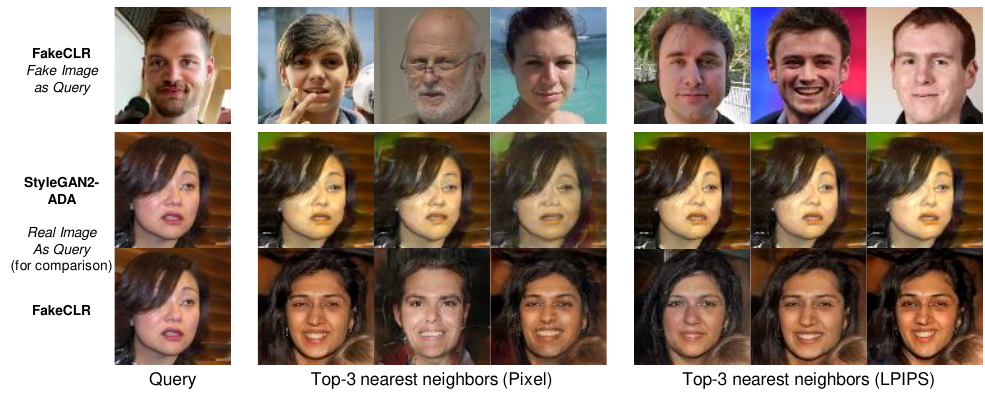}
	\end{center}
	\caption{Nearest neighbors in pixel space (left) and LPIPS feature space (right) on the FFHQ-100 dataset. Followed as DiffAugment \cite{zhao2020differentiable}, we select fake image as query in the first row. Furthermore, we select the real image as new query for comparing with previous methods in the bottom two rows.}
	\label{fig:nearest neighbors_FFHQ-100}
\end{figure}
\subsection{Comparisons on PPL and KID Metrics on FFHQ-100 and Obama Datasets}
We report the PPL and KID metrics on FFHQ-100 and Obama Datasets in Table \ref{fig:KID}. Both of them demonstrate the superiority of FakeCLR.

\begin{table}[h]
    \centering 
    \caption{{Comparison with previous methods over FFHQ-2K and Obama datasets: KID and PPL of $w$ space are reported. Results with $^*$ are the searched best results for each dataset, which are better than InsGen with the default setting.} }
    \begin{tabular}{c|c c|c c}
    \hline
    \multirow{2}{*}{Methods} & \multicolumn{2}{c|}{KID($\times 10^{-2}$)$\downarrow$} 
    & \multicolumn{2}{c}{PPL($w$)$\downarrow$} \\
    \cline{2-5}
   &Obama&FFHQ-2K &Obama&FFHQ-2K\\
    \hline
    ADA[10]&1.29&0.466&843&232\\
    APA[18]&1.12&0.562&836&220\\
    InsGen$^*$[8]&0.605&0.455&818&206\\
    FakeCLR (our)&\textbf{0.296}&\textbf{0.345}&\textbf{752}&\textbf{136}\\
    \hline
    \end{tabular}
    \label{fig:KID}
    \end{table}
\subsection{Generated Images on Few-shot Generation}
Fig. \ref{FIG:generated_imgaes_few} illustrates some randomly synthesized images on Obama (100), Grumpy Cat (100), Panda (100), AnimalFace Cat (160), and AnimalFace Dog (389) datasets.
\begin{figure*}
\setlength{\abovecaptionskip}{0.1cm}
    \setlength{\belowcaptionskip}{-0.3cm}
	\centering
	\includegraphics[scale=0.26]{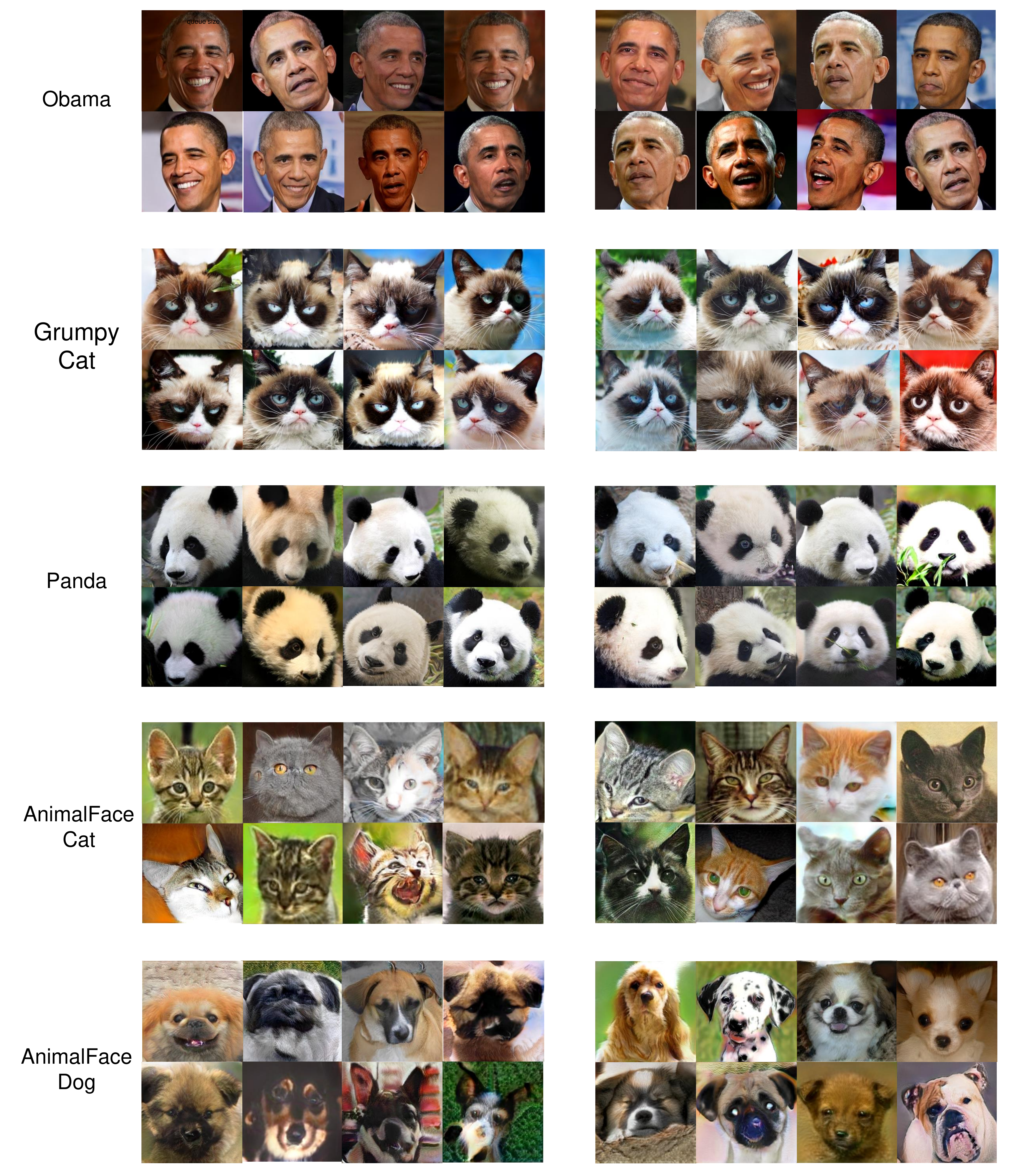}
	\caption{Generated images with different number of training images on few-shot dataset. All images are synthesized randomly without truncation.}
	\label{FIG:generated_imgaes_few}
\end{figure*}
\subsection{Experiments Considering Path Length Regularization}
Path length regularization in StyleGAN2 is another way to promote latent continuity. In this section, we consider the strength of Path length regularization in the FakeCLR. As shown in Table \ref{tab:pl_weight}, path length regularization boosts latent continuity, yet decreases the FID a little bit. Our FakeCLR show consistent improvement under different regularization settings.
\begin{table}[h]
\caption{FID and PPL on FFHQ-2K for different strength of PL regularization (All other experiments follow the default parameter PL weight=2). $^\dagger$ indicates that we adopt the fixed queue size = 200 for a fair comparison.}
    \centering 
\begin{tabular}{c|ccc|ccc} 
\hline  
\multirow{2}{*}{Methods}& \multicolumn{3}{c|}{FID}&\multicolumn{3}{c}{PPL ($w$)}\\
 \cline{2-7}
&2&5&10&2&5&10\\
\hline  
\ \ \ InsGen [8] \ \ \ &12.11&13.87&12.29&199&183&180\\
FakeCLR$^\dagger$&\ {\bf 10.40} \ &\ {\bf 10.78}\ &\ {\bf 11.32}\ &\ {\bf 171}\ &\ {\bf 160}\ &\ {\bf 145}\ \\
\hline
\end{tabular} 
 \label{tab:pl_weight}
\end{table}
\subsection{Experiments on other architectures and settings}
As shown in Table \ref{tab:cifar10}, we conducted experiments on CIFAR-10 (using full and $20\%$ training data) under the unconditional and conditional settings and StyleGAN2-ADA and BigGAN backbones. From both architectures and settings aspects, our FakeCLR can
improve the generation performance.
\begin{table}[h]
    \centering 
 \caption{{Comparisons of different architectures and settings (CIFAR-10 dataset). }}
    \setlength\tabcolsep{2pt}
    \begin{tabular}{c|c c|c c}
    \hline
    \multirow{2}{*}{Methods} & \multicolumn{2}{c|}{Conditional} 
    & \multicolumn{2}{c}{Unconditional} \\
    \cline{2-5}
    &$20\%$&full&$20\%$&full\\
    \cline{1-5}
    \ \ \ BigGAN-APA [18] \ \ \ &15.3&8.28&-&-\\
    BigGAN-FakeCLR&13.6&7.62&-&-\\
    StyleGAN-InsGen [8]&5.3&2.24&4.83&2.7\\
    StyleGAN-FakeCLR&\textbf{4.28}&\textbf{2.13}&\textbf{4.6}&\textbf{2.32}\\
    \hline
    \end{tabular}
 \label{tab:cifar10}
\end{table}
\begin{table}[h]
    \centering 
\caption{{FID on adding real samples into fake queue over FFHQ-2K dataset.}}
\begin{tabular}{c|c} 
\hline 
Methods&FID\\
\hline
Instance-perturbation (baseline) & \ \ \ \ \ \ \ {\bf 11.29} \ \ \ \ \ \ \ \\
Config A &12.46\\
Config B&13.71\\
\hline  
\end{tabular} 
 \label{tab:queue}
\end{table}
\subsection{What if adding real samples into the fake queue?}
Table \ref{tab:queue} presents some explorations of adding real samples into the fake queue of FakeCLR? \textit{Instance-perturbation} is selected as the baseline for a fair comparison. In Config A, we add real images from the beginning of training, 50$\%$ real+50$\%$ fake samples form the queue. Similarly, we add real images from the half iterations of training (Config B). 
Experiments demonstrate that adding real samples into our queue is not conducive to the final results. The reasons could be: i) Our FakeCLR aims to solve latent space discontinuity, while adding real samples may have side effects on learning the correlation among fake samples. ii) In terms of contrastive learning, real samples can be regarded as easy negative samples for fake samples, but contrastive learning usually prefers hard negative samples.
\end{document}